\documentclass{article} 
\usepackage{times}
\usepackage{hyperref,amsmath,amsthm,graphicx}
\usepackage{url}
\pdfoutput=1

\title{Steps Toward Deep Kernel Methods from Infinite Neural Networks}

\author{Tamir Hazan \\ tamir.hazan@gmail.com \\ University of Haifa \\ Carmel \\  Haifa 31905, Israel
\and Tommi Jaakkola \\ tommi@mit.edu \\ CSAIL, MIT \\ Cambridge \\ MA 02139, USA
}

\newcommand{\lrangle}[1]{\langle #1 \rangle}
\DeclareMathOperator{\erf}{erf}
\DeclareMathOperator{\relu}{ReLU}
\DeclareMathOperator{\step}{step}
\newtheorem{theorem}{Theorem}
\newtheorem{corollary}{Corollary}

\renewcommand{\[}{\begin{eqnarray}}
\renewcommand{\]}{\end{eqnarray}}

\begin{document}

\maketitle
\begin{abstract}
Contemporary deep neural networks exhibit impressive results on practical problems. These networks generalize well although their inherent capacity may extend significantly beyond the number of training examples. We analyze this behavior in the context of deep, infinite neural networks. We show that deep infinite layers are naturally aligned with Gaussian processes and kernel methods, and devise stochastic kernels that encode the information of these networks. We show that stability results apply despite the size, offering an explanation for their empirical success. 
\end{abstract}

\section{Introduction}
\label{sec:introduction}

Deep neural networks have become widely adopted for tasks ranging from image labeling in computer vision to parsing and machine translation in natural language processing. The networks in these tasks usually consist of an input layer, several semi-structured intermediate layers and an output layer. Surprisingly, as large, complex models, they appear easier to learn at scale, rendering state of the art performance (e.g., \cite{Hinton12}) with increasing amounts of data and computation. The setting poses new questions for learning since the number of parameters in these models, mostly residing in the deep layers, may be substantially larger than what could be supported by the training examples. Many expect such networks to overfit while in practice they (often) do not, and their decision boundaries appear smooth. Our work suggests an explanation for this behavior based on deep and \emph{infinitely} wide networks where the number of parameters is uncountably infinite.   

Neural networks with a single infinite intermediate layer have been considered by various works. \cite{Hornik93} show these networks are universal approximators and \cite{Neal95, Williams97, Cho09} explore their properties in the context of Gaussian processes and kernel methods. Unfortunately, since these networks interact linearly with the input layer, they are limited in their representation power. Moreover, the recent success of neural networks seems to rely on deep architecture while current infinite networks only encode the information of a single layer. Lastly, these works do not explain why learning the likelihood of infinite networks does not overfit, and the decision boundary of the learned network is simple. 

In our work we extend the framework of kernel methods for infinite networks to multiple layers. We introduce stochastic kernels that are derived from Gaussian processes and encode the information of two infinite layers. We also provide a generalization bound for these networks, based on stability of regularized loss minimization, and attribute the simplicity of the learned deep infinite network to the fast convergence of algorithms on our learning framework. 

We begin by introducing infinite neural networks with a single intermediate layer. We relate their learning units to integrals over functions in the Euclidean space with respect to the Gaussian distribution, as well as describe their connections to kernel functions. We subsequently construct the second layer and relate its learning units to expectations with respect to a Gaussian process. These expectations form stochastic kernel functions that encode the multilayer and infinitely wide neural network. Finally, we analyze the generalization properties of these networks and introduce a method to incorporate localities and non-linearities such as those arising from convolutional neural networks.

\section{Background}
\label{sec:background}

\begin{figure}
\centerline{
\includegraphics[width=1.5in]{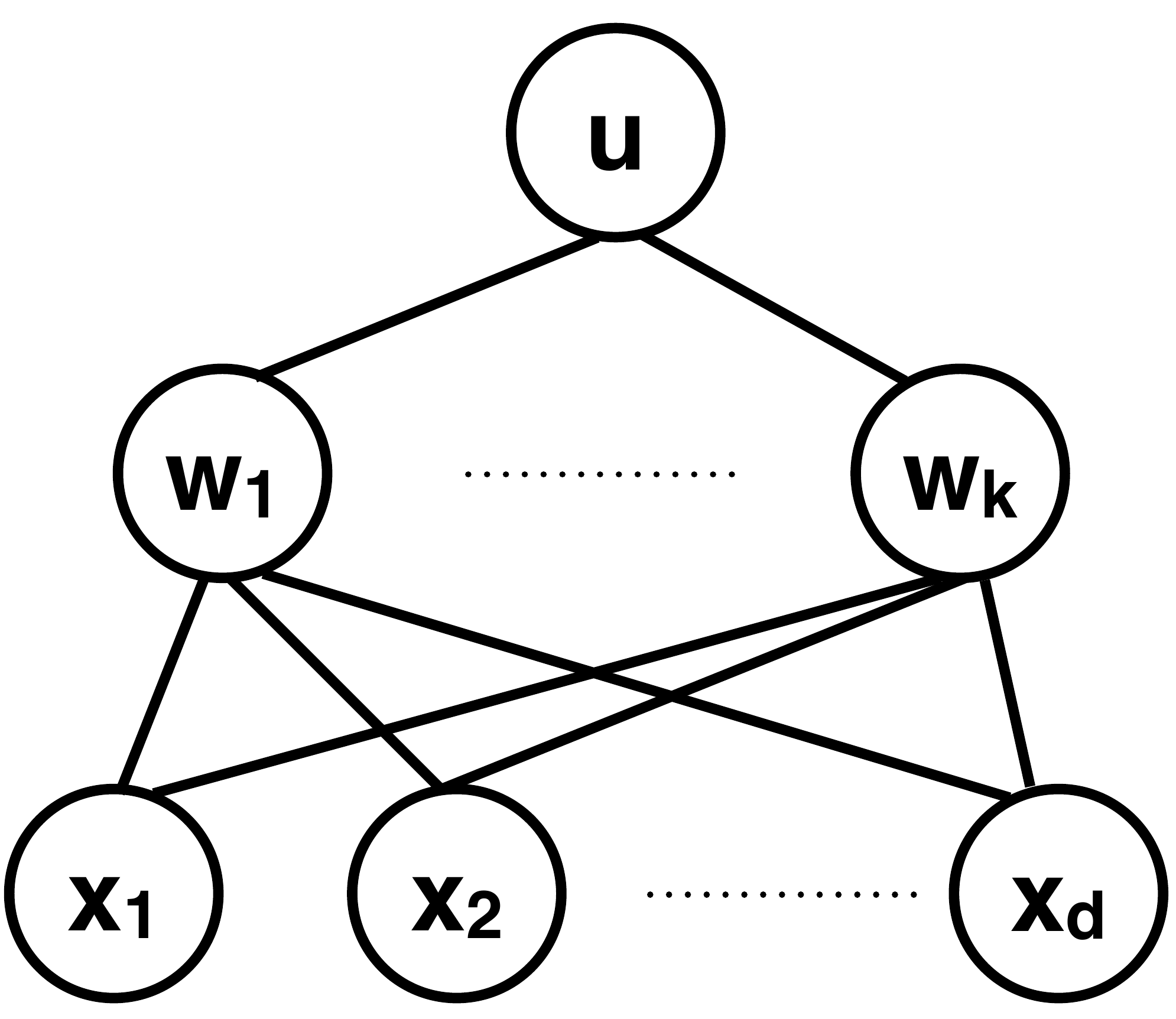} \hspace{0.5cm} 
\includegraphics[width=1.5in]{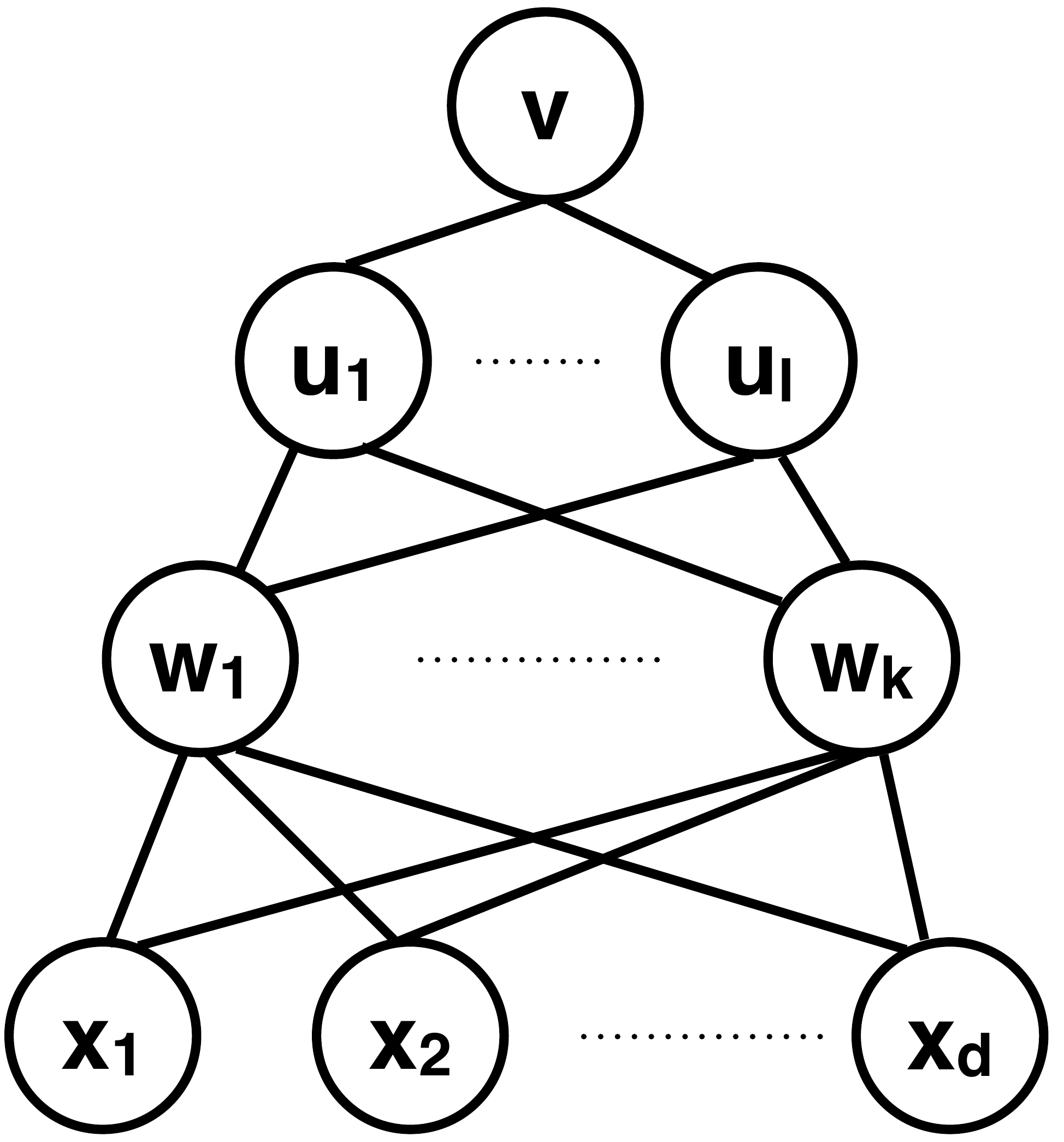} \hspace{0.5cm} 
\includegraphics[width=2in]{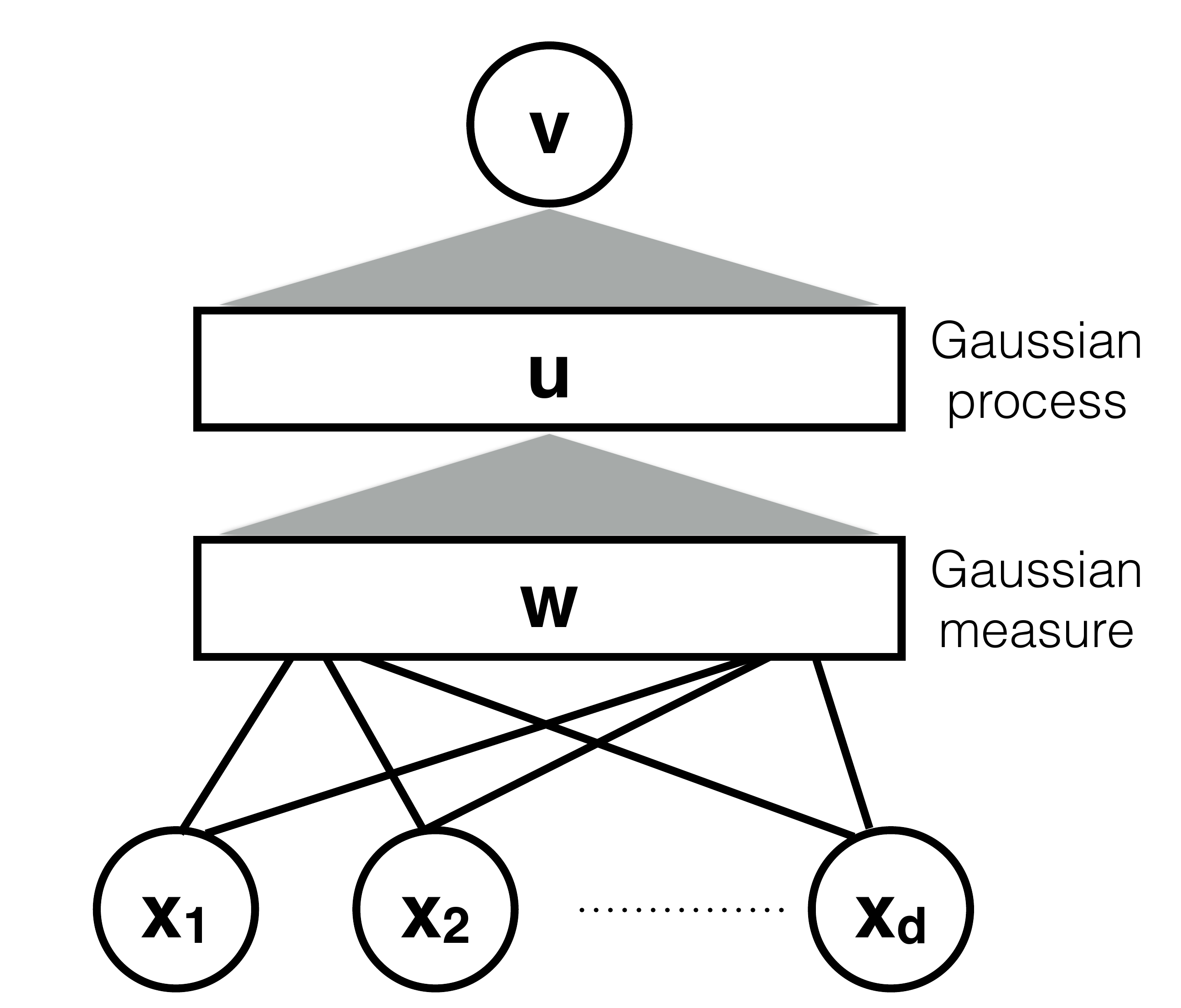}
}
\caption{\label{fig:nn} \small Left and middle images present finite neural networks with one and two intermediate layers, respectively. The right image depicts neural network with two infinitely wide intermediate layers, one is indexed by $w$ which consists of the functions $\phi_x(w)$ and the other indexed by $u$ and consists the functions $\psi_x(u)$. The functions $\phi_x(w)$ are associated with a Gaussian measure over $w$ and the functions $\psi_x(u)$ are associated with a Gaussian process over $u$.} 
\end{figure}

Neural networks form a successful framework for classification that imitate the activation function of neurons. Finite neural networks are usually described by a layered graph, see Figure \ref{fig:nn}. Its input layer consists of nodes that receive the input signal $x \in R^d$. Its subsequent layers consist of  parameters that encode the classification process: its intermediate layers consist of activation nodes. Each activation node rely on its parameter to produce a linear response $f(\lrangle{x,w})$ according to its parameters $w$. The function $f(t)$ is called an activation function or a transfer function and it introduces non-linearities to the network. Transfer functions imitate the neuron behavior, activating its value whenever its linear input $\lrangle{x,w}$ is high enough. There are various forms of non-linear transfer functions, e.g., step and rectified linear functions. 
Recently, the rectified linear function $\relu(t) = \max(0,t)$ was successfully used in neural networks as it carries the neuron signal better \cite{Hinton12}. Another popular transfer function is the step function $\step(t) = 1[t \ge 0]$ that attains the value one it $t \ge 0$ and zero otherwise. The output of a network with a single intermediate layer linearly weights the activations $f(\lrangle{x,w_1}), ..., f(\lrangle{x,w_k})$ with the output parameters $u_1,...,u_k$. Its classification is determined by the sign of $\sum_{i=1}^k u_i f(\lrangle{x,w_i})$.  

A classical result by Hornik asserts that networks with one intermediate layer are universal approximators when the number of activation units $k$ tends to infinity \cite{Hornik93}. Consequently, neural networks have been studied in the infinite setting \cite{Neal95, Williams97, Bengio05, Cho09, Liu14, Huang14}. In this setting there are infinitely many transfer functions $f(\lrangle{w,x})$ each of them is indexed by $w$. Summing over infinitely many transfer functions is formalized by integrating over possible $w$. Formally, we replace $\sum_i u_i f(\lrangle{x,w_i})$ with $\int u(w) f(\lrangle{x,w}) d \mu(w)$. The measure $\mu(w)$ may be any probability distribution over $R^d$ as long as this integral is finite, e.g, the Gaussian distribution $d \mu(w) = (2 \pi)^{-d/2} \exp(-\|w\|^2/2)$. 

When taking a discriminative approach, one learns a neural network that best describes the training data $S = \{(x_1,y_1), ..., (x_m,y_m)\}$, where $x_i$ is a data instance (e.g., an image or a sentence) and $y_i$ is its semantic label.  While learning an infinite network with a single intermediate layer, one needs to consider compact ways to represent the function $u(w)$. Kernel methods can be used for this task while representing the classifier by its dual \cite{Bengio05}. Particularly, the network's output is an inner product between $u(w)$ and an input-dependent function $\phi_x(w) = f(\lrangle{x,w})$ 
\[
\label{eq:G}
\lrangle{u, \phi_x}_\mu = \int u(w) f(\lrangle{x,w}) d \mu(w).
\]
Since $u(w)$ is trained over a finite space of feature functions it can be restricted without loss of generality to the linear span of the training feature functions $\phi_{x_1}(w), ..., \phi_{x_m}(w)$, namely $u(w) =  \sum_{i=1}^m \alpha_i \phi_{x_i}(w)$ for some real valued numbers $\alpha_1,...,\alpha_m$. Therefore, when evaluating the output value of the network $\lrangle{u, \phi_{x_j}}_\mu$ it suffices to compute the kernel entries
\[
k_f(x_i,x_j) = \lrangle{\phi_{x_i}, \phi_{x_j}}_\mu =  \int f(\lrangle{x_i,w}) f(\lrangle{x_j,w}) d \mu(w). \nonumber
\]
Various works have already computed the kernel function for different transfer functions with respect to the Gaussian measure, including the rectified linear and the sign function \cite{Williams97, Cho10}. In all these cases, the kernel has an analytic form although the features $\phi_x(w)$ are not finite vectors but functions over $R^d$. Explicitly, let $\rho_{i,j} = \lrangle{x_i,x_j}/\|x_i\| \|x_j\|$ then
\[
k_{\relu}(x_i,x_j) &=& \frac{\|x_i\| \|x_j\|}{\pi} \sin \big( \arccos (\rho_{i,j})   \big) + (\pi-\arccos (\rho_{i,j})) \rho_{i,j}. \nonumber \\
k_{\step}(x_i,x_j) &=& \pi-\arccos (\rho_{i,j}). \nonumber 
\] 

Whenever the measure is not Gaussian there is no analytic solution for the different kernels. Nevertheless, whenever the probability density function $d \mu(w)$ is log-concave (i.e., $\log(d \mu(w))$ is a concave function) then $\lrangle{x,w}$ is a log-concave function thus the sample complexity of the kernel function decays exponentially with the number of samples.

\section{Stochastic kernels for deep and infinitely wide neural networks}
\label{sec:deep}

A deep learning architecture considers multiple intermediate layers. Deep architectures have proven successful as they allow to express non-linearities easily. Unfortunately, when considering multiple infinite layers there are difficulties to represent the network parameters. Such difficulties do not appear when considering finite  layers since all parameters in all layers are vectors. However, when considering infinite layers, the parameters are functions (in the second intermediate layer) functions of functions (in the subsequent layer) and so on, see Figure \ref{fig:nn}. In the following we present the framework of learning with multiple intermediate layers. For the clarity of presentation we describe two intermediate layers. We refer to these networks as deep networks to differentiate them from the known networks with a single infinite layer. 

The main challenge in working with deep infinite networks is to establish the space in which the deep neurons exist. The neurons of the second intermediate layer take as input the functions $\phi_x(w)$, which are the output of the first intermediate layer, i.e., $\phi_x(w)$ for any $w \in R^d$. Therefore, each neuron in the second layer is a function $u:R^d \rightarrow R$ that weights its input values (which are $\phi_x(w)$ for any $w \in R^d$) in a linear manner $\lrangle{u,\phi_x}_{\mu}$. The output of each such neuron is the activation of the transfer function $\psi_x(u) = f(\lrangle{\phi_x,u}_{\mu})$. Therefore, the output layer of deep infinite network needs to take all its inputs, i.e., $\psi_x(u)$ for any function $u(\cdot)$, and weight their activation by $v(u)$. Thus the output layer computes the linear function $\lrangle{v,\psi_x}_{\nu}$. with respect the a measure $\nu$. Next, we determine the measure space of $v(u)$ in terms of stochastic processes.   

It is natural to consider the activation of neurons in the second intermediate layer $\lrangle{\phi_x, u}_{\mu}$ with respect to the measure $\mu(w)$ using probabilistic terms. This linear function is the covariance of two random variables $\lrangle{\phi_x, u}_{\mu} = E_{w \sim \mu} \big[ \phi_{x}(w) u(w)]$. Similarly, the activation of the output neuron is $\lrangle{v,\psi_x}_{\nu} = E_{u \sim \nu} \big[ \psi_{x}(u) v(u)]$. With this perspective, the functions $u:R^d \rightarrow R$ are chosen randomly according to the measure $\nu$. Equivalently, $\nu$ is a stochastic process. In our work we restrict ourselves to a Gaussian process, a stochastic process for which any finite collection of random variables $u(w_1), ..., u(w_k)$ has a multivariate Gaussian distribution. A Gaussian process is completely determined by its first and second order statistics. The mean function $\mu(w)$ of a Gaussian process is $E_{u \sim \nu}[u(w)]$. Its covariance function $C(w_1,w_2) = E_{u \sim \nu}[u(w_1),u(w_2)]$. We consider Gaussian process with zero mean function and a general covariance function, thus we denote $\nu = GP(C)$. 


To learn a deep infinite network that linearly separates the training examples it suffices use the stochastic kernel function: 
\[
\label{eq:k2}
k^{(2)}_f(x_i,x_j) = \lrangle{\psi_{x_i}, \psi_{x_j}}_\nu =  E_{u \sim GP(C)} \big[ f(\lrangle{\phi_{x_i},u}) f(\lrangle{\phi_{x_j},u}) \big]
\]
Recall that the first layer responses are the transfer functions $\phi_{x_i}(w) = f(\lrangle{w,x_i}), \phi_{x_j}(w) = f(\lrangle{w,x_j})$. Thus a stochastic kernel for deep infinite network averages non-linearities while considering their covariances. 

Although the Gaussian process has infinitely many random variables, its unique properties allows to compute the stochastic kernel function analytically. 
\begin{theorem}
\label{theorem:gp}
$$ k^{(2)}_f(x_i,x_j) = E_{(z_1,z_2) \sim N(0,\Sigma)} \big[ f(z_1) f(z_2) \big]$$ 
$z=(z_1,z_2)$ is a bivariate Gaussian random variable with zero mean and covariance matrix $\Sigma$:

\[
\label{eq:cov}
 \Sigma = E_{w_1,w_2}  
\Bigg( 
\begin{array}{cc} 
f(\lrangle{w_1,x_i}) C(w_1,w_2) f(\lrangle{w_2,x_i}) &  f(\lrangle{w_1,x_i}) C(w_1,w_2) f(\lrangle{w_2,x_j}) \\  \\
f(\lrangle{w_1,x_i}) C(w_1,w_2) f(\lrangle{w_2,x_j}) &  f(\lrangle{w_1,x_j}) C(w_1,w_2) f(\lrangle{w_2,x_j}) 
\end{array} 
\Bigg)
\]

$w_1,w_2$ are chosen independently from a $d-$dimensional multivariate Gaussian with zero mean and unit covariance, i.e., $N(0,I)$.  
\end{theorem}  
\begin{proof}
$z_1 = \lrangle{\phi_{x_i},u}$ is a Gaussian random variable with zero mean\footnote{This is a classical result and can be shown by working with the Riemann-Stieltjes integral, decomposing it to finite sums. Since any finite instantiation of a Gaussian process is a multivariate Gaussian random variable with zero mean, the Riemann-Stieltjes sum is also a Gaussian random variable, thus the limit (using the characteristic function) is also Gaussian.} Similarly, $z_2 = \lrangle{\phi_{x_j},u}$ is a Gaussian random variable and both $z_1,z_2$ are jointly Gaussian. Thus $z=(z_1,z_2)$ is a bivariate Gaussian random variable with zero mean and some covariance matrix $\Sigma$. The expected value of a Gaussian process reduces to     
$$E_{u \sim GP(C)} \big[ f(\lrangle{\phi_{x_i},u}) f(\lrangle{\phi_{x_j},u}) \big] = E_{(z_1,z_2) \sim N(0,\Sigma)} \big[ f(z_1) f(z_2) \big].$$
The covariance matrix of $\Sigma$ is a $2 \times 2$ matrix whose $(r,s)$ entry is $E_{z \in N(0,\Sigma)} E[z_r z_s]$. Recall that $z_1 = E_{w} [ \phi_{x_i}(w) u(w) ]$ and that $\Sigma_{11} = E[z_1^2]$, then 
\begin{eqnarray*}
\Sigma_{11} &=& E_{u} \Big[ E_{w_1}  [\phi_{x_i}(w_1) u(w_1) ] \cdot E_{w_2}  [\phi_{x_i}(w_2) u(w_2) ]  \Big] \\
&=& E_{w_1,w_2} \Big[ \phi_{x_i}(w_1)  \phi_{x_j}(w_2)  E_{u} [u(w_1) u(w_2) ]  \Big] = E_{w_1,w_2} \Big[ \phi_{x_i}(w_1)  \phi_{x_j}(w_2)  C(w_1, w_2)  \Big]. 
\end{eqnarray*}
We used Fubini's theorem to change the order of integration. The values of $\Sigma_{r,s}$ then follow in the same manner, while recalling that $\phi_{x_i}(w_1) = f(\lrangle{x_i,w_1})$ and $\phi_{x_j}(w_2) = f(\lrangle{x_j,w_2})$ . 
\end{proof}

An important family of Gaussian processes is described by shift-invariant covariance functions, namely $C(w_1,w_2) = c(w_1-w_2)$. Bochner's theorem represents such covariance functions as  $E_{w,b} [g(\lrangle{w_1,w}+b) g(\lrangle{w_2,w}+b)]$, where $w$ is drawn from a distribution $\rho$ over $R^d$, $b$ is drawn from the uniform distribution over $[0,2\pi]$ and $g(t) = \sqrt{2} \cos(t)$ \cite{Rahimi07}. Whenever $\rho$ is known we are able to efficiently compute a stochastic kernel for deep infinite networks and shift-invariant covariance functions: 
\begin{corollary}
Let $C(w_1,w_2) = c(w_1-w_2)$ be a shift-invariant covariance function and let $\rho$ be its corresponding measure derived by Bochner's theorem. Consider the $6 \times 6$ covariance matrix 
$$\hat \Sigma =
\left( 
\begin{array}{ccc} 
\|x_i\|^2             & \lrangle{x_i,w} &  \lrangle{x_i,x_j}  \\ 
\lrangle{x_i,w}  &  \|w\|^2             &  \lrangle{w,x_j}   \\
\lrangle{x_i,x_j}& \lrangle{x_j,w} &  \|x_j\|^2 
\end{array} 
\right)
\otimes 
\left( 
\begin{array}{cc} 
1 & 0\\
0 & 1
\end{array} 
\right).
$$   
$A \otimes B$ the tensor product of two matrices. 
Let $g(t) = \sqrt{2} \cos(t)$ and assume that $b$ is drawn form the uniform distribution over $[0,2\pi]$ and $w$ is drawn according to $\rho$ and $\hat z \sim N(0,\hat \Sigma)$ is a multivariate Gaussian. Then the covariance matrix $\Sigma$ of the stochastic kernel for deep infinite neural networks $k^{(2)}_f(x_i,x_j) = E_{(z_1,z_2) \sim N(0,\Sigma)}[ f(z_1) f(z_2) ]$ is 
$$ \Sigma = E_{w, b, \hat z}  
\Bigg( 
\begin{array}{cc} 
f(\hat z_1) f(\hat z_2)  g(\hat z_3  + b) g(\hat z_4  + b) &  f(\hat z_1) f(\hat z_6) g(\hat z_3  + b) g(\hat z_4  + b)   \\  \\
f(\hat z_1) f(\hat z_6) g(\hat z_3  + b) g(\hat z_4  + b)  &  f(\hat z_5) f(\hat z_6) g(\hat z_3  + b) g(\hat z_4  + b) 
\end{array} 
\Bigg)$$   
 \end{corollary}
\begin{proof}
The entries of the covariance matrix of the stochastic kernel are derived in Equation (\ref{eq:cov}) using $\Sigma_{r,s} = h_{r,s}(\hat z_1,..., \hat z_6)$ where $z_1 = \lrangle{w_1,x_i}, z_2 = \lrangle{w_2,x_i}, z_3 = \lrangle{w_1,w}, z_4 = \lrangle{w_2,w}, z_5 = \lrangle{w_1,x_j}, z_6 = \lrangle{w_2,x_j}$. Since $w_1,w_2 \sim N(0,I)$ are independent then $\hat z$ is a multivariate Gaussian with zero mean and its distribution is fully determined by its covariance matrix $\hat \Sigma$. The corollary then follows by direct computation of the covariance matrix, e.g.,, $\hat \Sigma_{1,3} = E_{\hat z} [\hat z_1 \hat z_3] = \sum_{r,s} E_{w_1} [w_{1,r} w_{1,s} x_{i,r} w_s] = \sum_{r,s}  x_{i,r} w_s \cdot E_{w_1} [w_{1,r} w_{1,s}]$ and $E_{w_1} [w_{1,r} w_{1,s}] = 1[r=s]$ is the indicator function that equals one if $r=s$ and zero otherwise.  
\end{proof}
The ability to realize the measure $\rho$ that is suggested by Bochner's theorem determines the validity of this approach. Bochner's theorem relates a shift-invariance function $c(w_1-w_2)$ to Fourier transform, thus $\rho$ can be recovered by its inverse-transform. There are some special covariance functions for which this measure is known. For example, the covariance function $C(w_1,w_2) = \exp(-\|w_1-w_2\|_1)$ that relates to the Ornstein-Uhlenbeck Gaussian process can be computed using the Cauchy distribution $d \rho(w) = \prod_{i=1}^d (\pi(1+ w_i^2))^{-1}$. Whenever the covariance function defines a a squared exponential Gaussian process, $C(w_1,w_2) = \beta \exp(-\|w_1-w_2\|_2^2)$, the stochastic kernel for deep neural networks can be computed analytically. This follows from the observation that the Gaussian process couples the independent $d-$dimensional Gaussians random variables $w_1,w_2$ to a $2d-$dimensional Gaussian variable $w=(w_1,w_2)$ with correlation $\alpha$: 
\begin{corollary}
Consider the covariance function $C(w_1,w_2) = (1+2\alpha)^{1+d/2} \exp(-\alpha \|w_1 - w_2\|^2/2)$. Consider the $4 \times 4$ covariance matrix 
$$\hat \Sigma =
\left( 
\begin{array}{cc} 
\|x_i\|^2               & \lrangle{x_i,x_j} \\
\lrangle{x_i,x_j}   & \|x_j\|^2
\end{array} 
\right)
\otimes 
\left( 
\begin{array}{cc} 
1+\alpha & \alpha \\
\alpha     & 1+\alpha
\end{array} 
\right)
$$   
Then the covariance matrix $\Sigma$ of the stochastic kernel for deep neural network $k^{(2)}_f(x_i,x_j) = E_{(z_1,z_2) \sim N(0,\Sigma)}[ f(z_1) f(z_2) ]$ is 
$$ \Sigma = E_{\hat z \sim N(0,\hat \Sigma)}  
\Bigg( 
\begin{array}{cc} 
f(\hat z_1) f(\hat z_2)  &  f(\hat z_1) f(\hat z_4)   \\
f(\hat z_1) f(\hat z_4)  &  f(\hat z_3) f(\hat z_4)  
\end{array} 
\Bigg)$$   
Moreover, $k^{(2)}_{\relu}(x_i,x_j)$ and $k^{(2)}_{\step}(x_i,x_j)$ have analytic forms.  
\end{corollary}
\begin{proof}
Considering Equation (\ref{eq:cov}) we note that $g_{I}(w_1)g_I(w_2)C(w_1,w_2) = (1+2\alpha) g_{\hat \Sigma}(w)$, where $g_I(w_i)$ is the $d-$dimensional Gaussian density function $N(0,I)$ and $g_{\hat \Sigma}(w)$ is the $2d-$dimensional Gaussian density function $N(0,\tilde \Sigma)$. We denote by $A \otimes B$ the tensor product of two matrices, thus 
$$(1+2\alpha) \tilde \Sigma =
\left( 
\begin{array}{cc} 
1+\alpha & \alpha \\
\alpha     & 1+\alpha
\end{array} 
\right)
\otimes I_{d \times d} 
$$   
The form of $\hat \Sigma$ is attained when setting $\hat z_1 = \lrangle{w_1,x_i}$, $\hat z_2 = \lrangle{w_2,x_i}$, $\hat z_3 = \lrangle{w_1,x_j}$, $\hat z_4 = \lrangle{w_2,x_j}$. With this notation, the form of $\Sigma$ is a direct consequence of Equation (\ref{eq:cov}). 

To compute the entries of $\Sigma$ when $f(t) = \relu(t)$ we recall that whenever $z'_1, z'_2 \in N(0, \Sigma')$ with $\Sigma'_{11} = \sigma_1^2$, $\Sigma'_{12} = \Sigma'_{21} = \rho \sigma_1 \sigma_2$, $\Sigma'_{22} = \sigma_2^2$ then $E_{z'_1, \hat z'_2} [f(z'_1) f(z'_2)] = h(\sigma_1,\sigma_2, \rho)$ and $h(\sigma_1,\sigma_2, \rho) = \frac{\sigma_1 \sigma_2}{\pi} \sin(\arccos(\rho)) + \rho(\pi - \arccos(\rho))$. Then $\Sigma_{\relu} =$
$$ 
\Bigg( 
\begin{array}{cc} 
h(\sqrt{1+\alpha}\|x_i\|, \sqrt{1+\alpha}\|x_i\|,\frac{\alpha}{1+\alpha})  &  
h(\sqrt{1+\alpha}\|x_i\|, \sqrt{1+\alpha}\|x_j\|,\frac{\alpha}{1+\alpha} \cdot \frac{\lrangle{x_i,x_j}}{\|x_i| \|x_j\|})  \\
h(\sqrt{1+\alpha}\|x_i\|, \sqrt{1+\alpha}\|x_j\|,\frac{\alpha}{1+\alpha} \cdot \frac{\lrangle{x_i,x_j}}{\|x_i| \|x_j\|})  &  
h(\sqrt{1+\alpha}\|x_j\|, \sqrt{1+\alpha}\|x_j\|,\frac{\alpha}{1+\alpha})  
\end{array} 
\Bigg)$$   
Thus, $k^{(2)}_{\relu}(x_i,x_j) = E_{(z_1,z_2) \sim N(0,\Sigma)}[ f(z_1) f(z_2) ]$ is a recursive application of $h( \cdot)$ with the appropriate parameters.     

To compute the entries of $\Sigma$ when $f(t) = 1[t\ge 0]$ we recall that whenever $z'_1, z'_2 \in N(0, \Sigma')$ with $\Sigma'_{11} = \sigma_1^2$, $\Sigma'_{12} = \Sigma'_{21} = \rho \sigma_1 \sigma_2$, $\Sigma'_{22} = \sigma_2^2$ then $E_{z'_1, \hat z'_2} [f(z'_1) f(z'_2)] = h(\rho) = \pi - \arccos(\rho)$. Then 
$$ 
\Sigma_{\step} =
\Bigg( 
\begin{array}{cc} 
h(\frac{\alpha}{1+\alpha})  &  
h(\frac{\alpha}{1+\alpha} \cdot \frac{\lrangle{x_i,x_j}}{\|x_i| \|x_j\|})  \\
h(\frac{\alpha}{1+\alpha} \cdot \frac{\lrangle{x_i,x_j}}{\|x_i| \|x_j\|})  &
h(\frac{\alpha}{1+\alpha})  
\end{array} 
\Bigg)$$   
As before, $k^{(2)}_{\step}(x_i,x_j) = E_{(z_1,z_2) \sim N(0,\Sigma)}[ f(z_1) f(z_2) ]$ is a recursive application of $h( \cdot)$ with the appropriate parameters.     
\end{proof}

Deep neural networks are usually applied to multiclass problems, where there are more than two labels to classify. Thus the label space resides in a discrete set $y \in \{1,...,K\}$. For notational convenience we focus above on binary classification, where $y \in \{-1,1\}$ is determined by the sign of the output layer $\lrangle{v,\phi_x}$. In multiclass setting, a data-instance $x$ can belong to any of the $K$ classes. A standard extension of the above setting to multiclass learning is to introduce $K$ decision boundaries $v_1(u),...,v_k(u)$. Multiclass prediction is performed by choosing the decision which is most certain, i.e., $\arg \max_i \lrangle{v_i,\psi_x}$. In  the next section we describe the generalization properties of deep infinite neural networks in the multiclass setting.

\section{Deep infinite networks, generalization and experimental validation}
\label{sec:theory}

The practice of neural networks proves that they do not overfit, even when the number of learned parameters is orders of magnitude larger than the number of training examples. In the following we address this scenario while suggesting some insight for why infinite networks generalize well. We show that generalization is mostly dependent on the expressive power of the output layer, which is regularized by its capacity. Consider a multiclass deep infinite network $v_1(u),...,v_k(u)$ that classifies the functions $\psi_{x_i}(u)$ according to the most certain linear response function $y_v(x) = \arg \max_i \lrangle{v_i,\psi_{x}}$. Since each decision function $v_i(u)$ interacts linearly with the training data, it must be a finite sum of these functions, i.e., $v_i(u) = \sum_{j=1}^m \alpha_{i,j} \psi_{x_j}(u)$. Therefore, as long as the functions $\psi_{x_i}(u)$ are simple (e.g., truncated linear functions in the case of ReLU units), the capacity of the deep infinite network is limited by the size of the training data. Whenever there are stronger guarantees on the data, i.e., that the training data is separable with a margin, they translate to a stronger regularization on $v_k(u)$ that is derived from the passive-aggressive learner \cite{Crammer06}. To be more precise, we say that the data is separable when there are functions $v_1^*(u),...,v_k^*(u)$ that classifies correctly any data instance. Formally, for any data-label pair $(x,y)$ there holds $y=y_{v^*}(x)$. These data-label pairs are separated with a margin if $y = y_{v^*}(x)$ and $\lrangle{v^*_y,\psi_x}  \ge 1 + \max_{i \ne y} \lrangle{v^*_i,\psi_x}$. In this setting, the kernel version of the passive-aggressive algorithm ensures that $v_i(u) = \sum_{j=1}^{t} \alpha_{i,j} \psi_{x_j}(u)$, where $t \le R^2 \sum_i \|v^*_i\|^2$ and $\|\psi_{x}\|^2 \le R$. Thus, whenever there is a separation with margin and the training size $m \gg t$ the passive-aggressive analysis ensures that the deep learner has restricted capacity thus a simple form. 

Unfortunately, the separable setting rarely exists in practice. Nevertheless, deep learners perform well in the non-separable setting. Usually deep learning schemes use the logistic regression framework, that maximizes the conditional probability of the training data $S=\{(x_1,y_1),..., (x_m,y_m)\}$. The conditional probability follows the Gibbs distribution $p_v(y|x) = \exp(\lrangle{v_y,\psi_x})/Z(v)$ where $Z(v) = \sum_i \exp(\lrangle{v_i,\psi_x})$ is the partition function. Thus the parameters of the network are learned by the optimization program: 
\[
\label{eq:lr}
v^S = \arg \max_v \frac{1}{m} \sum_{i=1}^m \log p_v(y_i|x_i) + \frac{\lambda_m}{2} \sum_j \|v_j\|^2
\]
As this is an infinite convex program it is appealing to consider its dual. The dual program is $\min_{\alpha} \sum_{i,k} \alpha_{i,k} \log \alpha_{i,k} + \sum_{i=1}^m \|v_i(\alpha)\|^2 / 2 \lambda_m$ where $v(\alpha) = \sum_{i} (\psi_{x_i} - (\sum_k \alpha_{i,k} \psi_{x_k}))$ and $\sum_k \alpha_{i,k} = 1$. The dual program is smooth and strongly convex, therefore enjoys rapid convergence, i.e., with $O(\log(1/\epsilon))$ updates to the elements $\alpha$ a dual exponentiated coordinate  descent achieves an $\epsilon-$optimal dual solution \cite{Collins08}. Although this algorithm achieves a good primal solution in practice, it does not guarantee that $v(\alpha)$ is an $\epsilon-$optimal primal solution as well. Recently, many efficient algorithms were devised to achieve both dual and primal guarantee with $O(\log (1/\epsilon))$ steps (cf. \cite{Roux12}). These algorithms aggregate data points $\psi_{x_i}$ to their separators $v(u)$ therefore after a small number of steps a good, yet simple separator is reached. Said differently, although different separators may exist around $v^S(u)$ the algorithm outputs a fairly simple separator as it is regularized by an early stopping criterion.  

Considering the learning problem in Equation (\ref{eq:lr}) as a loss minimization task, it measures the average log-loss given training data. By the above, the empirical risk minimizer  $v^S$ is simple, i.e., it consists of $O(\log(1/\epsilon))$ functions $\psi_x(u)$. We turn to show that this simple empirical risk minimizer also generalizes well, it achieves a similar log-loss even when the data-label pairs are sampled from their true distribution in the world. 
\begin{theorem}
Assume that $\| \phi_x \| \le 1$ and that the training examples are sampled independently from the data-label generating distribution $(x,y) \sim D$. Denote log-risk by $L_D(v) = E_{(x,y) \sim D} p_{v}(y|x)$ and the empirical risk by $L_S(v) = \frac{1}{m} \sum_{i=1}^m \log p_{v}(y_i|x_i)$. Consider $v^S$ as defined in Equation (\ref{eq:lr}), then $| L_D(v^S) - L_S(v^S) | \le 1/m \lambda_m$. 
\end{theorem}
\begin{proof}
Generalization by stability for convex and Lipschitz loss functions with strongly convex regularizer was established in \cite{Bousquet02, Mukherjee06, Shalev10}. Although the technical details are obscured in some of these results, we rely on their derivations (specifically \cite{Bousquet02} Theorem 22 and \cite{Shalev10} Theorem 2). The benefit of working with stability is that its basic concepts, convexity and Lipschitz continuity, readily generalize to infinite spaces. To apply generalization via stability to multiclass logistic regression we note that $-\log p_v(y|x)$ is convex. Also, it is $1-$Lipschitz since its gradient is uniformly bounded by $1$ whenever $\|\psi_x\| \le 1$. 
\end{proof}
The regularization ratio $\lambda_m$ is chosen such that $m \lambda_m$ goes to zero as $m$ tends to infinity. The important conclusion of the above theorem is that infinite models does not necessarily overfit, as long as the infinite model interacts in a constrained manner with the data. In our case the infinite model is constrained by convexity and Lipschitz continuity. These two properties stabilize the learning procedure while ensuring that small changes in $v$ do not change the prediction by much. 

Next, we turn to experimentally validate our framework. The effectiveness of infinite network with a single infinite layer using the kernels $k_{\relu}(x_i,x_j)$ was already demonstrated by \cite{Cho09, Cho10, Cheng11}. Thus in the following we show that our stochastic kernels $k^{(2)}_{\relu}(x_i,x_j)$ with a squared exponential Gaussian process improves upon $k_{\relu}(x_i,x_j)$. We run our kernels over MNIST digit database. This dataset is the standard entry point of neural networks and kernel methods. 

Our stochastic kernel $k^{(2)}_{\relu}(x_i,x_j)$ was able to separate the training data completely with only $50$ iterations, while $k_{\relu}(x_i,x_j)$ that encodes a single infinite layer did not (it nearly separated all examples). This validates the assertion that the stochastic kernel is more expressive than the single layer kernel. As for test results, the average error over all digits is $1.7\%$ for the stochastic kernel and $1.9\%$ for the single layer kernel. Although the improvement is modest in terms of the overall success rate (only $0.2\%$) it might be insightful to compare it to the possible gain over the errors of the single layer kernel, namely  $0.2/1.9$ which is about $10\%$ gain. Lastly, since our kernel function is computed analytically, it is trained as fast as any kernel method.

\section{Non-linearities}
\label{sec:cnn}

The infinite layers, presented in Section \ref{sec:background} and Section \ref{sec:deep}, are limited in their expressive power. In the first intermediate layer, the inner product $\lrangle{x,w}$ is performed for any $w \in R^d$, while each parameter $w$ acts globally on all input entries $x \in R^d$ linearly. Similarly, in the second layer, the function $u(w)$ acts linearly and globally on every $\phi_x(w) = f(\lrangle{w,x})$. These interactions ignore spatial information in the vectors $x$ or the feature function $\phi_x(w)$, spatial information that is important in computer vision and language processing applications. Current deep learning architectures exploit spatial information using convolutions. These convolutions are applied to patches in an image, or equivalently to overlapping subsets of the data instance $x$, and recursively to their functions. These operations introduce important aspects of non-linearity and locality. Our approach can be extended to deal with such operations, thus able to increase the expressiveness of our approach to various non-linearities.  

To describe a convolution-based operation in the first intermediate layer, we transform the data instance $x \in R^d$ to subsets of its elements $x^{(1)},....,x^{(P)} \in R^{d_1}$, where $x^{(p)} \subset x$. For each such subset we learn infinitely many responses $w \in R^{d_1}$, while each response outputs $\phi_{x,p}(w) = f(\lrangle{w,x^{(p)}})$. Thus, $\phi_{x} = (\phi_{x,1},...,\phi_{x,P})$ is a $P-$dimensional function, $\phi_{x}:R^{d_1} \rightarrow R^P$. Note that in Section \ref{sec:background} the feature function $\phi_x$ mapped $R^d$ to $R$. 

Convolution based operations in the second intermediate layer may also be applied. The feature function $\phi_x$ is transformed to subsets of its elements $\phi_x^{(1)},..., \phi_x^{(Q)}$ where $\phi_x^{(q)} \subset \phi_{x}$, i.e., $\phi_x^{(q)}: R^{d_1} \rightarrow R^{d_2}$ that is attained by restricting to $\phi_{x}(w)$ to some of its coordinates. Each of these subsets is weighted by $u:R^{d_1} \rightarrow R^{d_2}$ and its resulting response is $\psi_{x,q}(u) = f(\lrangle{u,\phi_{x}{^{(q)}}})$, while $\lrangle{u,\phi_{x}{(q)}} = E_{w} [ \lrangle{u(w),\phi_{x}^{q}(w)} ]$ and the latter inner product $ \lrangle{u(w),\phi_{x}^{q}(w)}$ is between two vectors in $R^{d_2}$. 

The above two constructions show how to integrate convolution-type non-linearities in deep infinite networks. The appropriate kernels follow a straight forward derivation of these higher dimension constructions.

\section{Related work}
\label{sec:related}

Neural networks, kernel methods and Gaussian processes have had a significant impact on the machine learning community and a full exposition of these methods can be found in machine learning textbooks on neural networks \cite{Bengio09}, kernel methods \cite{Scholkopf99} and Gaussian processes \cite{Rasmussen06}. 

Neural networks are attracting a considerable attention in the last few years. Their practical success is unmatched in several machine learning applications (e.g., \cite{Hinton12}). In recent years it was possible to construct deep learning architectures with considerable number of parameters that is significantly larger than the number of training examples. Surprisingly, these networks avoid overfitting. Several machine learning theories were devised to explain how deep networks avoid overfitting based on dropouts (e.g., \cite{Wager13, Maaten13}). Our approach is different since we represent neural networks with significant amount of parameters as infinite networks with multiple layers. We encode the neurons responses in functions, while each layer increases the complexity of its functions, namely the first layer consists of functions over the Euclidean space and the second layer consists of Gaussian processes. We avoid overfitting since our algorithm achieves an almost optimal solution with a few steps, thus our resulting classifier is simple to represent and regularized by early stopping. We provide a generalization bound for our classifier based on stability \cite{Bousquet02, Mukherjee06, Shalev10}.   

Infinite neural networks were introduced by \cite{Neal95, Williams97} in the context of Bayesian learning. They analyze the predictive probability of a neural network with an infinitely wide intermediate layer. In particular, when the transfer function is bounded, this predictive probability converges to a Gaussian process. When resolving the covariance function of this process, \cite{Williams97} realized the kernel $k_{\erf}(x_i,x_j)$ along with other kernel functions. This work differs from ours in a few respects. First, our work does not consider the predictive probability of labels given data but rather we aim at maximizing the likelihood of infinitely wide layers, a task that initially was supposed to overfit and generalize poorly \cite{Williams97}. We establish the prediction power of our approach using stability. Second, we build on multiple intermediate layers while trying to analyze the success of deep learning architectures, as opposed to \cite{Neal95, Williams97}. Lastly, our work considers Gaussian processes differently than \cite{Williams97}. We use Gaussian process to define a measure over our second (e.g., deep) intermediate layer. 

More recently, researchers explored different algorithms to learn infinite neural networks \cite{Bengio05}. This work formulates learning an infinite network as an infinite convex program and devise an incremental algorithm that is based on its dual representation. \cite{Rahimi09} suggest to optimize an infinite networks with a single layer using randomization to decrease the computational complexity of the learning algorithm. Our work addresses other properties of learning infinite networks, mainly Gaussian processes for constructing multiple layers and analyze how infinite networks avoid overfitting. 

Kernel methods for infinite neural networks are further explored in \cite{Cho09, Cho10}. These works introduce the kernels $k_{\step}(x_i,x_j), k_{\relu}(x_i,x_j)$ along with other kernels thus augment the works of \cite{Neal95, Williams97}. Moreover, they introduce kernel composition approach to simulate deep architecture. Our work differs in the way we address and analyze deep architectures of infinite networks. We construct deep layers that use as input their previous layer using Gaussian processes. Connections between kernel methods and Gaussian processes were left as an open problem in \cite{Cho10}. We also introduce a way to incorporate non-linearities and invariances such as convolutional neural networks in our framework, another open problem raised by \cite{Cho10}. In addition, we analyze why our networks avoid overfitting. \cite{Cheng11} demonstrate the effectiveness of these kernels in language processing. 

In our work we provide unbiased estimate to our kernels in the second layer using Bochner's theorem. These kernels consider a shift invariant covariance function. \cite{Rahimi07} suggest the same estimator for kernel functions in the context of random features in kernel methods. \cite{Le13} suggested improved methods to reduce the variance of these estimates. Such estimators were recently used within kernel methods to match deep learning results in language processing \cite{Liu14, Huang14}.

\section{Discussion}
Deep neural networks are successful in machine learning applications although the number of their  parameters is orders of magnitude larger than the number of training examples. In this work we explain this behavior using deep infinite neural networks. We construct stochastic kernels that rely on Gaussian processes to encode such networks. We also explain how to introduce locality and non-linearity to such networks, similarly to the ones introduced by convolution neural networks. Lastly, we provide generalization bounds and regularity conditions that explain why these networks do not overfit. We present our framework with only two intermediate layers mainly for simplicity. It can be extended to any depth but the higher layers may not use non-linearities. The problem of finding analytic forms of stochastic kernels that encode arbitrarily deep layers with non-linearities is largely open.   

The work combines mostly separate areas in machine learning, including kernel methods, neural networks and Gaussian processes. As such, there are many direction that still need to be explored. Importantly, which non-linearities are significant in deep infinite networks and whether they can be learned from data. What probabilities best fit this framework and are there other properties of stochastic processes, besides of covariance, that control learning?

\bibliographystyle{plain}
\bibliography{HazJaa-arXiv15}

\begin{thebibliography}{10}

\bibitem{Bengio09}
Yoshua Bengio.
\newblock Learning deep architectures for ai.
\newblock {\em Foundations and trends{\textregistered} in Machine Learning},
  2(1):1--127, 2009.

\bibitem{Bengio05}
Yoshua Bengio, Nicolas~L Roux, Pascal Vincent, Olivier Delalleau, and Patrice
  Marcotte.
\newblock Convex neural networks.
\newblock In {\em NIPS}, pages 123--130, 2005.

\bibitem{Bousquet02}
Olivier Bousquet and Andr{\'e} Elisseeff.
\newblock Stability and generalization.
\newblock {\em The Journal of Machine Learning Research}, 2:499--526, 2002.

\bibitem{Cheng11}
Chih-Chieh Cheng and Brian Kingsbury.
\newblock Arccosine kernels: Acoustic modeling with infinite neural networks.
\newblock In {\em ICASSP}, pages 5200--5203, 2011.

\bibitem{Cho09}
Youngmin Cho and Lawrence~K Saul.
\newblock Kernel methods for deep learning.
\newblock In {\em Advances in neural information processing systems}, pages
  342--350, 2009.

\bibitem{Cho10}
Youngmin Cho and Lawrence~K Saul.
\newblock Large-margin classification in infinite neural networks.
\newblock {\em Neural computation}, 22(10):2678--2697, 2010.

\bibitem{Collins08}
M.~Collins, A.~Globerson, T.~Koo, X.~Carreras, and P.L. Bartlett.
\newblock {Exponentiated gradient algorithms for conditional random fields and
  max-margin markov networks}.
\newblock {\em The Journal of Machine Learning Research}, 9:1775--1822, 2008.

\bibitem{Crammer06}
Koby Crammer, Ofer Dekel, Joseph Keshet, Shai Shalev-Shwartz, and Yoram Singer.
\newblock Online passive-aggressive algorithms.
\newblock {\em The Journal of Machine Learning Research}, 7:551--585, 2006.

\bibitem{Hornik93}
Kurt Hornik.
\newblock Some new results on neural network approximation.
\newblock {\em Neural Networks}, 6(8):1069--1072, 1993.

\bibitem{Huang14}
Po-Sen Huang, Haim Avron, Tara~N Sainath, Vikas Sindhwani, and Bhuvana
  Ramabhadran.
\newblock Kernel methods match deep neural networks on timit.
\newblock In {\em ICASSP, 2014}, pages 205--209, 2014.

\bibitem{Hinton12}
Alex Krizhevsky, Ilya Sutskever, and Geoffrey~E Hinton.
\newblock Imagenet classification with deep convolutional neural networks.
\newblock In {\em NIPS}, pages 1097--1105, 2012.

\bibitem{Le13}
Quoc Le, Tam{\'a}s Sarl{\'o}s, and Alex Smola.
\newblock Fastfood--approximating kernel expansions in loglinear time.
\newblock In {\em Proceedings of the international conference on machine
  learning}, 2013.

\bibitem{Liu14}
Zhiyun Lu, Avner May, Kuan Liu, Alireza~Bagheri Garakani, Dong Guo,
  Aur{\'e}lien Bellet, Linxi Fan, Michael Collins, Brian Kingsbury, Michael
  Picheny, et~al.
\newblock How to scale up kernel methods to be as good as deep neural nets.
\newblock {\em arXiv preprint arXiv:1411.4000}, 2014.

\bibitem{Maaten13}
Laurens Maaten, Minmin Chen, Stephen Tyree, and Kilian~Q Weinberger.
\newblock Learning with marginalized corrupted features.
\newblock In {\em ICML}, pages 410--418, 2013.

\bibitem{Mukherjee06}
Sayan Mukherjee, Partha Niyogi, Tomaso Poggio, and Ryan Rifkin.
\newblock Learning theory: stability is sufficient for generalization and
  necessary and sufficient for consistency of empirical risk minimization.
\newblock {\em Advances in Computational Mathematics}, 25(1-3):161--193, 2006.

\bibitem{Neal95}
Radford~M Neal.
\newblock {\em Bayesian learning for neural networks}.
\newblock PhD thesis, University of Toronto, 1995.

\bibitem{Rahimi07}
Ali Rahimi and Benjamin Recht.
\newblock Random features for large-scale kernel machines.
\newblock In {\em NIPS}, pages 1177--1184, 2007.

\bibitem{Rahimi09}
Ali Rahimi and Benjamin Recht.
\newblock Weighted sums of random kitchen sinks: Replacing minimization with
  randomization in learning.
\newblock In {\em NIPS}, pages 1313--1320, 2009.

\bibitem{Rasmussen06}
Carl~Edward Rasmussen.
\newblock Gaussian processes for machine learning.
\newblock 2006.

\bibitem{Roux12}
Nicolas~L Roux, Mark Schmidt, and Francis~R Bach.
\newblock A stochastic gradient method with an exponential convergence \_rate
  for finite training sets.
\newblock In {\em NIPS}, pages 2663--2671, 2012.

\bibitem{Scholkopf99}
Bernhard Sch{\"o}lkopf, Christopher~JC Burges, and Alexander~J Smola.
\newblock {\em Advances in kernel methods: support vector learning}.
\newblock MIT press, 1999.

\bibitem{Shalev10}
Shai Shalev-Shwartz, Ohad Shamir, Nathan Srebro, and Karthik Sridharan.
\newblock Learnability, stability and uniform convergence.
\newblock {\em JMLR}, 11:2635--2670, 2010.

\bibitem{Wager13}
Stefan Wager, Sida Wang, and Percy~S Liang.
\newblock Dropout training as adaptive regularization.
\newblock In {\em NIPS}, pages 351--359, 2013.

\bibitem{Williams97}
Christopher~KI Williams.
\newblock Computing with infinite networks.
\newblock {\em Advances in neural information processing systems}, pages
  295--301, 1997.

\end{thebibliography}

\end{document}